\documentclass{article}

\usepackage[final]{nips_2018}
\PassOptionsToPackage{sort&compress}{natbib}
\setcitestyle{numbers,square,comma,sort}

\usepackage[utf8]{inputenc} 
\usepackage[T1]{fontenc}    
\usepackage{hyperref}       
\usepackage{url}            
\usepackage{graphicx}
\usepackage{booktabs}       
\usepackage{amsfonts}       
\usepackage{nicefrac}       
\usepackage{microtype}      
\usepackage{amsthm}
\usepackage{algorithm}
\usepackage{algorithmic}
\usepackage{amssymb}
\usepackage{mathtools}
\usepackage{tabularx}

\usepackage{float}
\usepackage{subfig}
\usepackage{array}
\usepackage{enumitem}
\usepackage[font=small,labelfont=bf]{caption}
\usepackage{amsmath}
\usepackage{color}
\usepackage{wrapfig}

\usepackage[toc,titletoc,page]{appendix}
\let\appendixpagenameorig\appendixpagename
\renewcommand{\appendixpagename}{\Large\appendixpagenameorig}

\usepackage{setspace}
\AtBeginDocument{%
  \addtolength\abovedisplayskip{-0.2\baselineskip}%
  \addtolength\belowdisplayskip{-0.2\baselineskip}%
} 

\usepackage{titlesec}
\titlespacing\subsection{0pt}{4pt plus 2pt minus 2pt}{3pt plus 2pt minus 2pt}

\newtheorem{theorem}{Theorem}
\newtheorem{definition}{Definition}

\newtheorem{lemma}{Lemma}
\newtheorem{proposition}{Proposition}
\makeatletter
\renewenvironment{proof}[1][\proofname]{\par
  \vspace{-\topsep}
  \pushQED{\qed}%
  \normalfont
  \topsep0pt \partopsep0pt 
  \trivlist
  \item[\hskip\labelsep
        \itshape
    #1\@addpunct{.}]\ignorespaces
}{%
  \popQED\endtrivlist\@endpefalse
}
\makeatother

\title{Link Prediction Based on Graph Neural Networks}

\author{
  Muhan Zhang\\
  Department of CSE\\
  Washington University in St. Louis\\
  \texttt{muhan@wustl.edu} \\
  \And
  Yixin Chen \\
  Department of CSE\\
  Washington University in St. Louis\\
  \texttt{chen@cse.wustl.edu} \\
}

\begin{document}

\maketitle

\begin{abstract}
Link prediction is a key problem for network-structured data. Link prediction heuristics use some score functions, such as common neighbors and Katz index, to measure the likelihood of links. They have obtained wide practical uses due to their simplicity, interpretability, and for some of them, scalability. However, every heuristic has a strong assumption on when two nodes are likely to link, which limits their effectiveness on networks where these assumptions fail. In this regard, a more reasonable way should be learning a suitable heuristic from a given network instead of using predefined ones. By extracting a local subgraph around each target link, we aim to learn a function mapping the subgraph patterns to link existence, thus automatically learning a ``heuristic'' that suits the current network. In this paper, we study this heuristic learning paradigm for link prediction. First, we develop a novel $\gamma$-decaying heuristic theory. The theory unifies a wide range of heuristics in a single framework, and proves that all these heuristics can be well approximated from local subgraphs. Our results show that local subgraphs reserve rich information related to link existence. Second, based on the $\gamma$-decaying theory, we propose a new method to learn heuristics from local subgraphs using a graph neural network (GNN). Its experimental results show unprecedented performance, working consistently well on a wide range of problems.

\end{abstract}
\section{Introduction}
Link prediction is to predict whether two nodes in a network are likely to have a link \cite{liben2007link}. Given the ubiquitous existence of networks, it has many applications such as friend recommendation \cite{adamic2003friends}, movie recommendation \cite{koren2009matrix}, knowledge graph completion \cite{nickel2016review}, and metabolic network reconstruction \cite{oyetunde2016boostgapfill}.

One class of simple yet effective approaches for link prediction is called heuristic methods. Heuristic methods compute some heuristic node similarity scores as the likelihood of links \cite{liben2007link,lu2011link}. Existing heuristics can be categorized based on the maximum hop of neighbors needed to calculate the score.
For example, common neighbors (CN) and preferential attachment (PA) \cite{barabasi1999emergence} are \textbf{first-order} heuristics, since they only involve the one-hop neighbors of two target nodes. Adamic-Adar (AA) and resource allocation (RA) \cite{zhou2009predicting} are \textbf{second-order} heuristics, as they are calculated from up to two-hop neighborhood of the target nodes. We define \textit{$h$-order heuristics} to be those heuristics which require knowing up to $h$-hop neighborhood of the target nodes. There are also some \textbf{high-order} heuristics which require knowing the entire network. Examples include Katz, rooted PageRank (PR) \cite{brin2012reprint}, and SimRank (SR) \cite{jeh2002simrank}. Table~\ref{heuristics} in Appendix \ref{appendix:features} summarizes eight popular heuristics.


Although working well in practice, heuristic methods have strong assumptions on when links may exist. For example, the common neighbor heuristic assumes that two nodes are more likely to connect if they have many common neighbors. This assumption may be correct in social networks, but is shown to fail in protein-protein interaction (PPI) networks -- two proteins sharing many common neighbors are actually less likely to interact \cite{kovacs2018network}.

In fact, the heuristics belong to a more generic class, namely \textit{graph structure features}. Graph structure features are those features located inside the observed node and edge structures of the network, which can be calculated directly from the graph. Since heuristics can be viewed as  predefined graph structure features, a natural idea is to automatically learn such features from the network. \citet{zhang2017weisfeiler} first studied this problem. They extract local enclosing subgraphs around links as the training data, and use a fully-connected neural network to learn which enclosing subgraphs correspond to link existence. Their method called Weisfeiler-Lehman Neural Machine (WLNM) has achieved state-of-the-art link prediction performance. The \textit{enclosing subgraph} for a node pair $(x,y)$ is the subgraph induced from the network by the union of $x$ and $y$'s neighbors up to $h$ hops. Figure \ref{f1} illustrates the 1-hop enclosing subgraphs for $(A,B)$ and $(C,D)$. These enclosing subgraphs are very informative for link prediction -- all first-order heuristics such as common neighbors can be directly calculated from the 1-hop enclosing subgraphs. 


However, it is shown that high-order heuristics such as rooted PageRank and Katz often have much better performance than first and second-order ones \cite{lu2011link}. To effectively learn good high-order features, it seems that we need a very large hop number $h$ so that the enclosing subgraph becomes the entire network. This results in unaffordable time and memory consumption for most practical networks. But do we really need such a large $h$ to learn high-order heuristics? 

Fortunately, as our first contribution, we show that we do not necessarily need a very large $h$ to learn high-order graph structure features. We dive into the inherent mechanisms of link prediction heuristics, and find that most high-order heuristics can be unified by a \textit{$\gamma$-decaying theory}. We prove that, under mild conditions, any $\gamma$-decaying heuristic can be effectively approximated from an $h$-hop enclosing subgraph, where the approximation error decreases at least exponentially with $h$. This means that we can safely use even a small $h$ to learn good high-order features. It also implies that the ``effective order'' of these high-order heuristics is not that high.



Based on our theoretical results, we propose a novel link prediction framework, SEAL, to learn general graph structure features from \textbf{local} enclosing subgraphs. SEAL fixes multiple drawbacks of WLNM. First, a graph neural network (GNN) \cite{bruna2013spectral,duvenaud2015convolutional,kipf2016semi,niepert2016learning,zhang2018end} is used to replace the fully-connected neural network in WLNM, which enables better graph feature learning ability. Second, SEAL permits learning from not only subgraph structures, but also latent and explicit node features, thus absorbing multiple types of information. We empirically verified its much improved performance.

Our contributions are summarized as follows. 1) We present a new theory for learning link prediction heuristics, justifying learning from \textbf{local} subgraphs instead of entire networks. 2) We propose SEAL, a novel link prediction framework based on GNN (illustrated in Figure \ref{f1}). SEAL outperforms all heuristic methods, latent feature methods, and recent network embedding methods by large margins. SEAL also outperforms the previous state-of-the-art method, WLNM.

\begin{figure}[tp]
\centering
\subfloat{\includegraphics[width=0.95\textwidth]{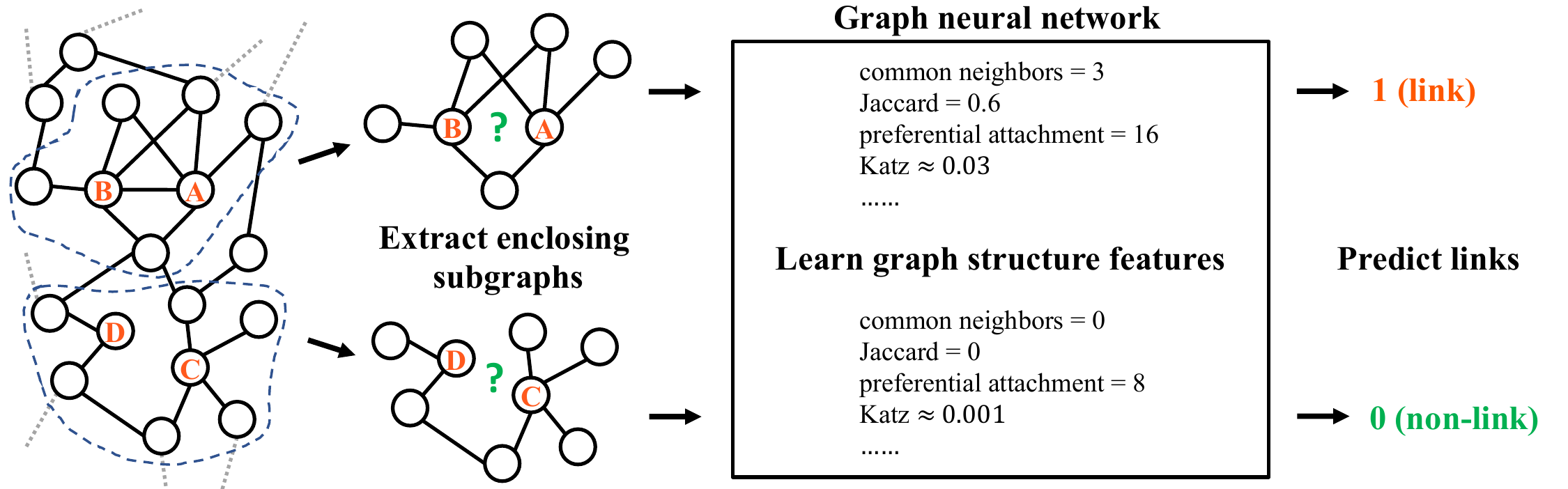}}
\caption{\small The SEAL framework. For each target link, SEAL extracts a local enclosing subgraph around it, and uses a GNN to learn general graph structure features for link prediction. Note that the heuristics listed inside the box are just for illustration -- the learned features may be completely different from existing heuristics.}
\label{f1}
\end{figure}

\section{Preliminaries}

\noindent \textbf{Notations~~}
Let $G = (V, E)$ be an undirected graph, where $V$ is the set of vertices and $E \subseteq V \times V$ is the set of observed links. Its adjacency matrix is $A$, where $A_{i,j} = 1$ if $(i,j) \in E$ and $A_{i,j} = 0$ otherwise. 
For any nodes $x, y \in V$, let $\Gamma(x)$ be the 1-hop neighbors of $x$, and $d(x,y)$ be the shortest path distance between $x$ and $y$. A walk $w = \langle v_0, \cdots, v_k \rangle$ is a sequence of nodes with $(v_i, v_{i+1}) \in E$. We use $|\langle v_0, \cdots, v_k \rangle|$ to denote the length of the walk $w$, which is $k$ here.

\noindent \textbf{Latent features and explicit features~~}
Besides graph structure features, latent features and explicit features are also studied for link prediction. \textbf{Latent feature methods} \cite{koren2009matrix,airoldi2008mixed,perozzi2014deepwalk,grover2016node2vec} factorize some matrix representations of the network to learn a low-dimensional latent representation/embedding for each node. Examples include matrix factorization \cite{koren2009matrix} and stochastic block model \cite{airoldi2008mixed} etc. Recently, a number of network embedding techniques have been proposed, such as DeepWalk \cite{perozzi2014deepwalk}, LINE \cite{tang2015line} and node2vec \cite{grover2016node2vec}, which are also latent feature methods since they implicitly factorize some matrices too \cite{qiu2017network}. 
\textbf{Explicit features} are often available in the form of node attributes, describing all kinds of side information about individual nodes. It is shown that combining graph structure features with latent features and explicit features can improve the performance \cite{nickel2014reducing,zhao2017leveraging}.

\noindent\textbf{Graph neural networks~~} Graph neural network (GNN) is a new type of neural network for learning over graphs \cite{bruna2013spectral,duvenaud2015convolutional,kipf2016semi,niepert2016learning,li2015gated,dai2016discriminative}). Here, we only briefly introduce the components of a GNN since this paper is not about GNN innovations but is a novel application of GNN. 
A GNN usually consists of 1) \textit{graph convolution layers} which extract local substructure features for individual nodes, and 2) a \textit{graph aggregation layer} which aggregates node-level features into a graph-level feature vector. Many graph convolution layers can be unified into a message passing framework \cite{gilmer2017neural}. 


\noindent\textbf{Supervised heuristic learning~~} 
There are some previous attempts to learn supervised heuristics for link prediction. The closest work to ours is the Weisfeiler-Lehman Neural Machine (WLNM) \cite{zhang2017weisfeiler}, which also learns from local subgraphs. However, WLNM has several drawbacks. Firstly, WLNM trains a fully-connected neural network on the subgraphs' adjacency matrices. Since fully-connected neural networks only accept fixed-size tensors as input, WLNM requires truncating different subgraphs to the same size, which may lose much structural information. Secondly, due to the limitation of adjacency matrix representations, WLNM cannot learn from latent or explicit features. Thirdly, theoretical justifications are also missing. We include more discussion on WLNM in Appendix \ref{appendix:baselines}.
Another related line of research is to train a supervised learning model on different heuristics' combination. For example, the path ranking algorithm \cite{lao2010relational} trains logistic regression on different path types' probabilities to predict relations in knowledge graphs. \citet{nickel2014reducing} propose to incorporate heuristic features into tensor factorization models. However, these models still rely on predefined heuristics -- they cannot learn general graph structure features.

\section{A theory for unifying link prediction heuristics} 
In this section, we aim to understand deeper the mechanisms behind various link prediction heuristics, and thus motivating the idea of learning heuristics from local subgraphs. Due to the large number of graph learning techniques, note that we are not concerned with the generalization error of a particular method, but focus on the information reserved in the subgraphs for calculating existing heuristics. 
\begin{definition} \textbf{(Enclosing subgraph)}
For a graph $G = (V, E)$, given two nodes $x, y \in V$, the $h$-hop enclosing subgraph for $(x,y)$ is the subgraph $G^h_{x,y}$ induced from $G$ by the set of nodes $\{~i~ |~ d(i,x) \leq h ~\text{or}~ d(i,y)\leq h ~\}$.
\end{definition}
\vspace{-5pt}




The enclosing subgraph describes the ``$h$-hop surrounding environment" of $(x,y)$. Since $G^h_{x,y}$ contains all $h$-hop neighbors of $x$ and $y$, we naturally have the following theorem. 

\begin{theorem}
Any $h$-order heuristic for $(x, y)$ can be accurately calculated from $G^h_{x,y}$.
\label{t1}
\end{theorem}
For example, a $2$-hop enclosing subgraph will contain all the information needed to calculate any first and second-order heuristics. 
However, although first and second-order heuristics are well covered by local enclosing subgraphs, an extremely large $h$ seems to be still needed for learning high-order heuristics. Surprisingly, our following analysis shows that \textbf{learning high-order heuristics is also feasible with a small $h$}. We support this first by defining the $\gamma$-decaying heuristic. We will show that under certain conditions, a $\gamma$-decaying heuristic can be very well approximated from the $h$-hop enclosing subgraph. Moreover, we will show that almost all well-known high-order heuristics can be unified into this $\gamma$-decaying heuristic framework.


\begin{definition} \textbf{($\gamma$-decaying heuristic)}
A $\gamma$-decaying heuristic for $(x,y)$ has the following form:
\begin{equation}
\mathcal{H}(x,y) = \eta \sum_{l=1}^\infty \gamma^l f(x,y,l),
\end{equation}
where $\gamma$ is a decaying factor between 0 and 1, $\eta$ is a positive constant or a positive function of $\gamma$ that is upper bounded by a constant, $f$ is a nonnegative function of $x,y,l$ under the the given network.
\end{definition}

Next, we will show that under certain conditions, a $\gamma$-decaying heuristic can be approximated from an $h$-hop enclosing subgraph, and the approximation error decreases at least exponentially with $h$.

\begin{theorem}
Given a $\gamma$-decaying heuristic $\mathcal{H}(x,y) = \eta \sum_{l=1}^\infty \gamma^l f(x,y,l)$, if $f(x,y,l)$ satisfies:
\vspace{-5pt}
\begin{itemize}[leftmargin=*]
\item(property 1) $f(x,y,l) \leq \lambda^l$ where $\lambda < \frac{1}{\gamma}$; and
\vspace{-5pt}
\item(property 2) $f(x,y,l)$ is calculable from $G^h_{x,y}$ for $l=1,2,\cdots, g(h)$, where $g(h)\!=\!ah\! +\! b$ with $a,b \in \mathbb{N}$ and $a>0$, 
\end{itemize}
\vspace{-5pt}
then $\mathcal{H}(x,y)$ can be approximated from $G^h_{x,y}$ and the approximation error decreases at least exponentially with $h$.
\label{expdecay}
\end{theorem}
\begin{proof}
We can approximate such a $\gamma$-decaying heuristic by summing over its first $g(h)$ terms.
\begin{equation}
\widetilde{\mathcal{H}}(x,y) := \eta \sum_{l=1}^{g(h)} \gamma^l f(x,y,l).
\end{equation}
The approximation error can be bounded as follows.
\[
|\mathcal{H}(x,y) - \widetilde{\mathcal{H}}(x,y) | = \eta \sum_{l=g(h)+1}^{\infty} \gamma^l f(x,y,l) \leq \eta \sum_{l=ah+b+1}^{\infty} \gamma^l \lambda^l = \eta (\gamma\lambda)^{ah+b+1} (1-\gamma\lambda)^{-1}. \qedhere
\]
\end{proof}
In practice, a small $\gamma\lambda$ and a large $a$ lead to a faster decreasing speed. Next we will prove that three popular high-order heuristics: Katz, rooted PageRank and SimRank, are all $\gamma$-decaying heuristics which satisfy the properties in Theorem \ref{expdecay}. First, we need the following lemma.
\begin{lemma}
Any walk between $x$ and $y$ with length $l \leq 2h+1$ is included in $G^h_{x,y}$.
\label{l1}
\end{lemma}
\begin{proof}
Given any walk $w = \langle x, v_1, \cdots, v_{l-1}, y \rangle$ with length $l$, we will show that every node $v_i$ is included in $G^h_{x,y}$. Consider any $v_i$. Assume $d(v_i,x) \geq h+1$ and $d(v_i,y) \geq h+1$. Then, $2h+1 \geq l = |\langle x, v_1,\cdots,v_{i}\rangle| + |\langle v_{i},\cdots,v_{l-1},y\rangle| \geq d(v_i,x) + d(v_i,y) \geq 2h+2$, a contradiction. Thus, $d(v_i,x) \le h$ or $d(v_i,y) \le h$. By the definition of $G^h_{x,y}$, $v_i$ must be included in $G^h_{x,y}$. 
\end{proof}
Next we will analyze Katz, rooted PageRank and SimRank one by one.

\subsection{Katz index}

The Katz index \cite{katz1953new} for $(x, y)$ is defined as
\begin{align}
\text{Katz}_{x,y} = \sum_{l = 1}^{\infty} \beta^l  {|\text{walks}^{\langle l\rangle}(x,y)|} = \sum_{l=1}^\infty \beta^l [A^l]_{x,y},
\label{katz1}
\end{align}
where $\text{walks}^{\langle l\rangle}(x,y)$ is the set of length-$l$ walks between $x$ and $y$, and $A^l$ is the $l^{\text{th}}$ power of the adjacency matrix of the network. Katz index sums over the collection of all walks between $x$ and $y$ where a walk of length $l$ is damped by $\beta^l$ ($0 < \beta < 1$), giving more weight to shorter walks.

Katz index is directly defined in the form of a $\gamma$-decaying heuristic with $\eta=1,\gamma=\beta$, and $f(x,y,l) = |\text{walks}^{\langle l\rangle}(x,y)|$. According to Lemma \ref{l1}, $|\text{walks}^{\langle l\rangle}(x,y)|$ is calculable from $G^h_{x,y}$ for $l\leq 2h+1$, thus property 2 in Theorem \ref{expdecay} is satisfied. Now we show when property 1 is satisfied.

\begin{proposition}
For any nodes $i,j$, $[A^l]_{i,j}$ is bounded by $d^{l}$, where $d$ is the maximum node degree of the network.
\label{l2}
\end{proposition}
\begin{proof}
We prove it by induction. When $l=1$, $A_{i,j} \leq d$ for any $(i,j)$. Thus the base case is correct. Now, assume by induction  that $[A^l]_{i,j} \leq d^{l}$ for any $(i,j)$, we have
\[
[A^{l+1}]_{i,j} = \sum_{k=1}^{|V|} [A^l]_{i,k}A_{k,j}
\leq d^{l} \sum_{k=1}^{|V|} A_{k,j}
\leq d^{l} d 
= d^{l+1}.  \qedhere
\]
\end{proof}

Taking $\lambda = d$, we can see that whenever $d < 1/\beta$, the Katz index will satisfy property 1 in Theorem \ref{expdecay}. In practice, the damping factor $\beta$ is often set to very small values like 5E-4 \cite{liben2007link}, which implies that Katz can be very well approximated from the $h$-hop enclosing subgraph.


\subsection{PageRank}

The rooted PageRank for node $x$ calculates the stationary distribution of a random walker starting at $x$, who iteratively moves to a random neighbor of its current position with probability $\alpha$ or returns to $x$ with probability $1-\alpha$. Let $\pi_x$ denote the stationary distribution vector. Let $[\pi_{x}]_i$ denote the probability that the random walker is at node $i$ under the stationary distribution. 

Let $P$ be the transition matrix with $P_{i,j} = \frac{1}{|\Gamma(v_j)|}$ if $(i,j) \in E$ and $P_{i,j} = 0$ otherwise. Let $\mathbf{e}_x$ be a vector with the $x^{\text{th}}$ element being $1$ and others being $0$. The stationary distribution satisfies
\begin{align}
\pi_x &= \alpha P\pi_x + (1-\alpha)\mathbf{e}_x.
\end{align}

When used for link prediction, the score for $(x,y)$ is given by $[\pi_{x}]_y$ (or $[\pi_{x}]_y + [\pi_{y}]_x$ for symmetry). To show that rooted PageRank is a $\gamma$-decaying heuristic, we introduce the \textit{inverse P-distance} theory \cite{jeh2003scaling}, which states that $[\pi_{x}]_y$ can be equivalently written as follows:
\begin{equation}
[\pi_{x}]_y = (1-\alpha)\sum_{w: x \rightsquigarrow y} P[w]  \alpha^{\text{len}(w)},
\end{equation}
where the summation is taken over all walks $w$ starting at $x$ and ending at $y$ (possibly touching $x$ and $y$ multiple times). For a walk $w = \langle v_0, v_1, \cdots, v_k \rangle$, $\text{len}(w) := |\langle v_0, v_1, \cdots, v_k \rangle|$ is the length of the walk. The term $P[w]$ is defined as $\prod_{i=0}^{k-1} \frac{1}{|\Gamma(v_i)|}$, which can be interpreted as the probability of traveling $w$. Now we have the following theorem.

\begin{theorem}
The rooted PageRank heuristic is a $\gamma$-decaying heuristic which satisfies the properties in Theorem \ref{expdecay}.
\end{theorem}
\begin{proof}
We first write $[\pi_{x}]_y$ in the following form.
\begin{align}
[\pi_{x}]_y = (1-\alpha)\sum_{l=1}^{\infty}\sum_{\substack{w: x \rightsquigarrow y \\ \text{len}(w) = l }} P[w]  \alpha^{l}.
\end{align}
Defining $f(x,y,l) := \sum_{\substack{w: x \rightsquigarrow y \\ \text{len}(w) = l }} P[w]$ leads to the form of a $\gamma$-decaying heuristic. Note that $f(x,y,l)$ is the probability that a random walker starting at $x$ stops at $y$ with exactly $l$ steps, which satisfies $\sum_{z\in V}f(x,z,l) = 1$. Thus, $f(x,y,l) \leq 1 < \frac{1}{\alpha}$ (property 1). According to Lemma \ref{l1}, $f(x,y,l)$ is also calculable from $G^h_{x,y}$ for $l\leq 2h+1$ (property 2).
\end{proof}

\subsection{SimRank}

The SimRank score \cite{jeh2002simrank} is motivated by the intuition that two nodes are similar if their neighbors are also similar. It is defined in the following recursive way: if $x=y$, then $s(x,y) := 1$; otherwise,
\begin{equation}
s(x,y) := \gamma \frac{\sum_{a\in \Gamma(x)}\sum_{b\in \Gamma(y)} s(a,b)}{|\Gamma(x)| \cdot |\Gamma(y)|}
\end{equation}
where $\gamma$ is a constant between 0 and 1. According to \cite{jeh2002simrank}, SimRank has an equivalent definition:
\begin{equation}
s(x,y) = \sum_{w: (x,y) \multimap (z,z)} \!\!\!P[w]  \gamma^{\text{len}(w)},
\end{equation}
where $w: (x,y) \multimap (z,z)$ denotes all simultaneous walks such that one walk starts at $x$, the other walk starts at $y$, and they first meet at \textbf{any} vertex $z$. For a simultaneous walk $w = \langle (v_0, u_0), \cdots, (v_k, u_k) \rangle$, $\text{len}(w) = k$ is the length of the walk. The term $P[w]$ is similarly defined as $\prod_{i=0}^{k-1} \frac{1}{|\Gamma(v_i)||\Gamma(u_i)|}$, describing the probability of this walk. Now we have the following theorem.

\begin{theorem}
SimRank is a $\gamma$-decaying heuristic which satisfies the properties in Theorem \ref{expdecay}.
\end{theorem}
\begin{proof}
We write $s(x,y)$ as follows.
\begin{align}
s(x,y) = \sum_{l=1}^{\infty} \sum_{\substack{w: (x,y) \multimap (z,z)\\  \text{len}(w) = l}} P[w]  \gamma^{l},
\end{align}
Defining $f(x,y,l) := \sum_{\substack{w: (x,y) \multimap (z,z)\\  \text{len}(w) = l}} P[w]$ reveals that SimRank is a $\gamma$-decaying heuristic. Note that $f(x,y,l) \leq 1 < \frac{1}{\gamma}$. It is easy to see that $f(x,y,l)$ is also calculable from $G^h_{x,y}$ for $l\leq h$.
\end{proof}

\noindent \textbf{Discussion~~}
There exist several other high-order heuristics based on path counting or random walk \cite{lu2011link} which can be as well incorporated into the $\gamma$-decaying heuristic framework. We omit the analysis here. 
Our results reveal that most high-order heuristics inherently share the same $\gamma$-decaying heuristic form, and thus can be effectively approximated from an $h$-hop enclosing subgraph with exponentially smaller approximation error. We believe the ubiquity of $\gamma$-decaying heuristics is not by accident -- it implies that a successful link prediction heuristic is better to put exponentially smaller weight on structures far away from the target, as remote parts of the network intuitively make little contribution to link existence.
Our results build the foundation for learning heuristics from local subgraphs, as they imply that local enclosing subgraphs already \textbf{contain enough information to learn good graph structure features} for link prediction which is much desired considering learning from the entire network is often infeasible. To summarize, from the small enclosing subgraphs extracted around links, we are able to accurately calculate first and second-order heuristics, and approximate a wide range of high-order heuristics with small errors. Therefore, given adequate feature learning ability of the model used, learning from such enclosing subgraphs is expected to achieve performance at least as good as a wide range of heuristics. There is some related work which empirically verifies that local methods can often estimate PageRank and SimRank well \cite{chen2004local,jia2010local}. Another related theoretical work \cite{bar2008local} establishes a condition of $h$ to achieve some fixed approximation error for ordinary PageRank.

\section{SEAL: An implemetation of the theory using GNN}




In this section, we describe our SEAL framework for link prediction. SEAL does not restrict the learned features to be in some particular forms such as $\gamma$-decaying heuristics, but instead learns general graph structure features for link prediction.
It contains three steps: 1) enclosing subgraph extraction, 2) node information matrix construction, and 3) GNN learning.
Given a network, we aim to learn automatically a ``heuristic'' that best explains the link formations. Motivated by the theoretical results, this function takes local enclosing subgraphs around links as input, and output how likely the links exist. To learn such a function, we train a graph neural network (GNN) over the enclosing subgraphs. Thus, the first step in SEAL is to extract enclosing subgraphs for a set of sampled positive links (observed) and a set of sampled negative links (unobserved) to construct the training data. 

A GNN typically takes $(A,X)$ as input, where $A$ (with slight abuse of notation) is the adjacency matrix of the input enclosing subgraph, $X$ is the \textit{node information matrix} each row of which corresponds to a node's feature vector. The second step in SEAL is to construct the node information matrix $X$ for each enclosing subgraph. This step is crucial for training a successful GNN link prediction model. In the following, we discuss this key step. The node information matrix $X$ in SEAL has three components: structural node labels, node embeddings and node attributes.

\subsection{Node labeling}
The first component in $X$ is each node's structural label. A node labeling is function $f_l: V \rightarrow \mathbb{N}$ which assigns an integer label $f_l(i)$ to every node $i$ in the enclosing subgraph. The purpose is to use different labels to \textbf{mark nodes' different roles} in an enclosing subgraph:
1) The center nodes $x$ and $y$ are the target nodes between which the link is located. 2) Nodes with different relative positions to the center have different structural importance to the link. 
A proper node labeling should mark such differences. If we do not mark such differences, GNNs will not be able to \textbf{tell where are the target nodes between which a link existence should be predicted}, and lose structural information. 

Our node labeling method is derived from the following criteria: 1) The two target nodes $x$ and $y$ always have the distinctive label ``$1$''. 2) Nodes $i$ and $j$ have the same label if $d(i,x) = d(j,x)$ and $d(i,y) = d(j,y)$. The second criterion is because, intuitively, a node $i$'s topological position within an enclosing subgraph can be described by its \textit{radius} with respect to the two center nodes, namely $(d(i,x), d(i,y))$. Thus, we let nodes on the same orbit have the same label, so that the node labels can reflect nodes' relative positions and structural importance within subgraphs.

Based on the above criteria, we propose a \textit{Double-Radius Node Labeling} (DRNL) as follows. First, assign label 1 to $x$ and $y$. Then, for any node $i$ with $(d(i,x), d(i,y)) = (1,1)$, assign label $f_l(i)=2$. Nodes with radius $(1,2)$ or $(2,1)$ get label 3. Nodes with radius $(1,3)$ or $(3,1)$ get 4. Nodes with $(2,2)$ get 5. Nodes with $(1,4)$ or $(4,1)$ get 6. Nodes with $(2,3)$ or $(3,2)$ get 7. So on and so forth. In other words, we iteratively assign larger labels to nodes with a larger radius w.r.t. both center nodes, where the label $f_l(i)$ and the double-radius $(d(i,x),d(i,y))$ satisfy

1) if $d(i,x) + d(i,y) \neq d(j,x) + d(j,y)$, then $d(i,x) + d(i,y) < d(j,x) + d(j,y) \Leftrightarrow  f_l(i) < f_l(j)$;

2) if $d(i,x) + d(i,y) = d(j,x) + d(j,y)$, then $d(i,x)d(i,y) < d(j,x)d(j,y) \Leftrightarrow  f_l(i) < f_l(j)$. 
 
One advantage of DRNL is that it has a perfect hashing function
\begin{align}
f_l(i) = 1 + min(d_x, d_y) + (d / 2)[(d / 2) + (d \% 2) - 1],
\label{hashing}
\end{align}
where $d_x := d(i,x)$, $d_y := d(i,y)$, $d := d_x + d_y$, $(d / 2)$ and $(d \% 2)$ are the integer quotient and remainder of $d$ divided by $2$, respectively. This perfect hashing allows fast closed-form computations.

For nodes with $d(i,x)=\infty$ or $d(i,y)=\infty$, we give them a null label 0. Note that DRNL is not the only possible way of node labeling, but we empirically verified its better performance than no labeling and other naive labelings. We discuss more about node labeling in Appendix \ref{appendix:labeling}. After getting the labels, we use their one-hot encoding vectors to construct $X$.

\subsection{Incorporating latent and explicit features}
Other than the structural node labels, the node information matrix $X$ also provides an opportunity to include latent and explicit features. By concatenating each node's embedding/attribute vector to its corresponding row in $X$, we can make SEAL simultaneously learn from all three types of features.

Generating the node embeddings for SEAL is nontrivial. Suppose we are given the observed network $G=(V,E)$, a set of sampled positive training links $E_p \subseteq E$, and a set of sampled negative training links $E_n$ with $E_n \cap E = \varnothing$. If we directly generate node embeddings on $G$, the node embeddings will record the link existence information of the training links (since $E_p \subseteq E$). We observed that GNNs can quickly find out such link existence information and optimize by only fitting this part of information. This results in bad generalization performance in our experiments.
Our trick is to temporally add $E_n$ into $E$, and generate the embeddings on $G' = (V, E \cup E_n)$. This way, the positive and negative training links will have the same link existence information recorded in the embeddings, so that GNN cannot classify links by only fitting this part of information. 
We empirically verified the much improved performance of this trick to SEAL. We name this trick \textit{negative injection}. 

We name our proposed framework \textbf{SEAL} (learning from Subgraphs, Embeddings and Attributes for Link prediction), emphasizing its ability to jointly learn from three types of features.


\section{Experimental results}
We conduct extensive experiments to evaluate SEAL. Our results show that SEAL is a superb and robust framework for link prediction, achieving unprecedentedly strong performance on various networks. We use AUC and average precision (AP) as evaluation metrics. We run all experiments for 10 times and report the average AUC results and standard deviations. We leave the the AP and time results in Appendix \ref{appendix:results}.
SEAL is flexible with what GNN or node embeddings to use. Thus, we choose a recent architecture DGCNN \cite{zhang2018end} as the default GNN, and node2vec \cite{grover2016node2vec} as the default embeddings. 
The code and data are available at https://github.com/muhanzhang/SEAL.

\noindent\textbf{Datasets~~} The eight datasets used are: USAir, NS, PB, Yeast, C.ele, Power, Router, and E.coli (please see Appendix \ref{appendix:datasets} for details). We randomly remove 10\% existing links from each dataset as positive testing data. Following a standard manner of learning-based link prediction, we randomly sample the same number of nonexistent links (unconnected node pairs) as negative testing data. We use the remaining 90\% existing links as well as the same number of additionally sampled nonexistent links to construct the training data.


\noindent\textbf{Comparison to heuristic methods~~} We first compare SEAL with methods that only use graph structure features.  We include eight popular heuristics (shown in Appendix \ref{appendix:features}, Table \ref{heuristics}): common neighbors (CN), Jaccard, preferential attachment (PA), Adamic-Adar (AA), resource allocation (RA), Katz, PageRank (PR), and SimRank (SR). We additionally include Ensemble (ENS) which trains a logistic regression classifier on the eight heuristic scores. We also include two heuristic learning methods: Weisfeiler-Lehman graph kernel (WLK) \cite{shervashidze2011weisfeiler} and WLNM \cite{zhang2017weisfeiler}, which also learn from (truncated) enclosing subgraphs. We omit path ranking methods \cite{lao2010relational} as well as other recent methods which are specifically designed for knowledge graphs or recommender systems \cite{nickel2014reducing,monti2017geometric}. As all the baselines only use graph structure features, we restrict SEAL to not include any latent or explicit features. In SEAL, the hop number $h$ is an important hyperparameter. Here, we select $h$ only from $\{1,2\}$, since on one hand we empirically verified that the performance typically does not increase after $h\geq 3$, which validates our theoretical results that the most useful information is within local structures. On the other hand, even $h=3$ sometimes results in very large subgraphs if a hub node is included. This raises the idea of sampling nodes in subgraphs, which we leave to future work.
The selection principle is very simple: If the second-order heuristic AA outperforms the first-order heuristic CN on 10\% validation data, then we choose $h=2$; otherwise we choose $h=1$. For datasets PB and E.coli, we consistently use $h=1$ to fit into the memory. We include more details about the baselines and hyperparameters in Appendix \ref{appendix:baselines}.

\newcolumntype{F}{>{\centering\arraybackslash}p{2.5em}}
\newcolumntype{G}{>{\raggedright\arraybackslash}p{2.8em}}
\newcolumntype{H}{>{\centering\arraybackslash}p{5.2em}}

\begin{table*}[th]
\centering
\caption{\small Comparison with heuristic methods (AUC).}
\resizebox{1\textwidth}{!}{
\begin{tabular}{@{}l |c*{12}{c}}
\toprule
\textbf{Data} &CN     & Jaccard   & PA     & AA     & RA     & Katz    & PR     & SR  &ENS   & WLK            & WLNM            & \textbf{SEAL}     \\ \midrule
USAir  & 93.80$\pm$1.22 & 89.79$\pm$1.61 & 88.84$\pm$1.45 & 95.06$\pm$1.03 & 95.77$\pm$0.92 & 92.88$\pm$1.42 & 94.67$\pm$1.08 & 78.89$\pm$2.31 & 88.96$\pm$1.44 & \textbf{96.63}$\pm$0.73 & 95.95$\pm$1.10 & \textbf{96.62}$\pm$0.72 \\
NS     & 94.42$\pm$0.95 & 94.43$\pm$0.93 & 68.65$\pm$2.03 & 94.45$\pm$0.93 & 94.45$\pm$0.93 & 94.85$\pm$1.10 & 94.89$\pm$1.08 & 94.79$\pm$1.08 & 97.64$\pm$0.25 & 98.57$\pm$0.51 & 98.61$\pm$0.49 & \textbf{98.85}$\pm$0.47 \\
PB     & 92.04$\pm$0.35 & 87.41$\pm$0.39 & 90.14$\pm$0.45 & 92.36$\pm$0.34 & 92.46$\pm$0.37 & 92.92$\pm$0.35 & 93.54$\pm$0.41 & 77.08$\pm$0.80 & 90.15$\pm$0.45 & 93.83$\pm$0.59 & 93.49$\pm$0.47 & \textbf{94.72}$\pm$0.46 \\
Yeast  & 89.37$\pm$0.61 & 89.32$\pm$0.60 & 82.20$\pm$1.02 & 89.43$\pm$0.62 & 89.45$\pm$0.62 & 92.24$\pm$0.61 & 92.76$\pm$0.55 & 91.49$\pm$0.57 & 82.36$\pm$1.02 & 95.86$\pm$0.54 & 95.62$\pm$0.52 & \textbf{97.91}$\pm$0.52 \\
C.ele  & 85.13$\pm$1.61 & 80.19$\pm$1.64 & 74.79$\pm$2.04 & 86.95$\pm$1.40 & 87.49$\pm$1.41 & 86.34$\pm$1.89 & \textbf{90.32}$\pm$1.49 & 77.07$\pm$2.00 & 74.94$\pm$2.04 & 89.72$\pm$1.67 & 86.18$\pm$1.72 & \textbf{90.30}$\pm$1.35 \\
Power  & 58.80$\pm$0.88 & 58.79$\pm$0.88 & 44.33$\pm$1.02 & 58.79$\pm$0.88 & 58.79$\pm$0.88 & 65.39$\pm$1.59 & 66.00$\pm$1.59 & 76.15$\pm$1.06 & 79.52$\pm$1.78 & 82.41$\pm$3.43 & 84.76$\pm$0.98 & \textbf{87.61}$\pm$1.57 \\
Router & 56.43$\pm$0.52 & 56.40$\pm$0.52 & 47.58$\pm$1.47 & 56.43$\pm$0.51 & 56.43$\pm$0.51 & 38.62$\pm$1.35 & 38.76$\pm$1.39 & 37.40$\pm$1.27 & 47.58$\pm$1.48 & 87.42$\pm$2.08 & 94.41$\pm$0.88 & \textbf{96.38}$\pm$1.45 \\
E.coli & 93.71$\pm$0.39 & 81.31$\pm$0.61 & 91.82$\pm$0.58 & 95.36$\pm$0.34 & 95.95$\pm$0.35 & 93.50$\pm$0.44 & 95.57$\pm$0.44 & 62.49$\pm$1.43 & 91.89$\pm$0.58 & 96.94$\pm$0.29 & 97.21$\pm$0.27 & \textbf{97.64}$\pm$0.22 \\
\bottomrule
\end{tabular}
}
\label{t3}
\end{table*}

Table \ref{t3} shows the results. Firstly, we observe that methods which learn from enclosing subgraphs (WLK, WLNM and SEAL) generally perform much better than predefined heuristics. This indicates that the learned ``heuristics'' are better at capturing the network properties than manually designed ones. Among learning-based methods, SEAL has the best performance, demonstrating GNN's superior graph feature learning ability over graph kernels and fully-connected neural networks. From the results on Power and Router, we can see that although existing heuristics perform similarly to random guess, learning-based methods still maintain high performance. This suggests that we can even discover new ``heuristics'' for networks where no existing heuristics work.

\begin{wraptable}{L}{0.61\textwidth}
\centering
\caption{\small Comparison with latent feature methods (AUC).}
\resizebox{0.602\textwidth}{!}{
\begin{tabular}{@{}l |c*{7}{c}}
\toprule
\textbf{Data}   & MF    & SBM   & N2V   & LINE  & SPC   & VGAE  & \textbf{SEAL } \\ \midrule
USAir  & 94.08$\pm$0.80 & 94.85$\pm$1.14 & 91.44$\pm$1.78 & 81.47$\pm$10.71 & 74.22$\pm$3.11 & 89.28$\pm$1.99 & \textbf{97.09}$\pm$0.70 \\ 
NS     & 74.55$\pm$4.34 & 92.30$\pm$2.26 & 91.52$\pm$1.28 & 80.63$\pm$1.90  & 89.94$\pm$2.39 & 94.04$\pm$1.64 & \textbf{97.71}$\pm$0.93 \\
PB     & 94.30$\pm$0.53 & 93.90$\pm$0.42 & 85.79$\pm$0.78 & 76.95$\pm$2.76  & 83.96$\pm$0.86 & 90.70$\pm$0.53 & \textbf{95.01}$\pm$0.34 \\
Yeast  & 90.28$\pm$0.69 & 91.41$\pm$0.60 & 93.67$\pm$0.46 & 87.45$\pm$3.33  & 93.25$\pm$0.40 & 93.88$\pm$0.21 & \textbf{97.20}$\pm$0.64 \\
C.ele  & 85.90$\pm$1.74 & 86.48$\pm$2.60 & 84.11$\pm$1.27 & 69.21$\pm$3.14  & 51.90$\pm$2.57 & 81.80$\pm$2.18 & \textbf{89.54}$\pm$2.04 \\
Power  & 50.63$\pm$1.10 & 66.57$\pm$2.05 & 76.22$\pm$0.92 & 55.63$\pm$1.47  & \textbf{91.78}$\pm$0.61 & 71.20$\pm$1.65 & 84.18$\pm$1.82 \\
Router & 78.03$\pm$1.63 & 85.65$\pm$1.93 & 65.46$\pm$0.86 & 67.15$\pm$2.10  & 68.79$\pm$2.42 & 61.51$\pm$1.22 & \textbf{95.68}$\pm$1.22 \\
E.coli & 93.76$\pm$0.56 & 93.82$\pm$0.41 & 90.82$\pm$1.49 & 82.38$\pm$2.19  & 94.92$\pm$0.32 & 90.81$\pm$0.63 & \textbf{97.22}$\pm$0.28 \\ \bottomrule
\end{tabular}
}
\label{t4}
\end{wraptable}

\noindent\textbf{Comparison to latent feature methods~~} Next we compare SEAL with six state-of-the-art latent feature methods: matrix factorization (MF), stochastic block model (SBM) \cite{airoldi2008mixed}, node2vec (N2V) \cite{grover2016node2vec}, LINE \cite{tang2015line}, spectral clustering (SPC), and variational graph auto-encoder (VGAE) \cite{kipf2016variational}. Among them, VGAE uses a GNN too. Please note the difference between VGAE and SEAL: VGAE uses a node-level GNN to learn node embeddings that best reconstruct the network, while SEAL uses a graph-level GNN to classify enclosing subgraphs. Therefore, VGAE still belongs to latent feature methods. For SEAL, we additionally include the 128-dimensional node2vec embeddings in the node information matrix $X$. Since the datasets do not have node attributes, explicit features are not included. 

Table \ref{t4} shows the results. As we can see, SEAL shows significant improvement over latent feature methods. One reason is that SEAL learns from both graph structures and latent features simultaneously, thus augmenting those methods that only use latent features.
We observe that SEAL with node2vec embeddings outperforms pure node2vec by large margins. This implies that network embeddings alone may not be able to capture the most useful link prediction information located in the local structures. 
It is also interesting that compared to SEAL without node2vec embeddings (Table \ref{t3}), joint learning does not always improve the performance.
More experiments and discussion are included in Appendix~\ref{appendix:results}.

\vspace{-5pt}
\section{Conclusions}
Learning link prediction heuristics automatically is a new field. In this paper, we presented theoretical justifications for learning from local enclosing subgraphs. In particular, we proposed a $\gamma$-decaying theory to unify a wide range of high-order heuristics and prove their approximability from local subgraphs. Motivated by the theory, we proposed a novel link prediction framework, SEAL, to simultaneously learn from local enclosing subgraphs, embeddings and attributes based on graph neural networks. Experimentally we showed that SEAL achieved unprecedentedly strong performance by comparing to various heuristics, latent feature methods, and network embedding algorithms. We hope SEAL can not only inspire link prediction research, but also open up new directions for other relational machine learning problems such as knowledge graph completion and recommender systems.



\subsubsection*{Acknowledgments}
The work is supported in part by the III-1526012 and SCH-1622678 grants from the National Science Foundation and grant 1R21HS024581 from the National Institute of Health. 

{\small
\bibliography{references}}

\newpage
\begin{appendices}

\section{More about the three types of features for link prediction}\label{appendix:features}
In this section, we discuss more about the difference among the three types commonly used features for link prediction: graph structure features, latent features, and explicit features.

\textbf{Graph structure features} locate inside the observed node and edge structures of the network, which can be directly observed and computed. Link prediction heuristics belong to graph structure features. We show eight popular heuristics in Table \ref{heuristics}. In addition to link prediction heuristics, node centrality scores (degree, closeness, betweenness, PageRank, eigenvector, hubs etc.), graphlets, network motifs etc. all belong to graph structure features. Although effective in many domains, these predefined graph structure features are handcrafted -- they only capture a small set of structure patterns, lacking the ability to express general structure patterns underlying different networks. Considering deep neural networks' success in feature learning, a natural question to ask is whether we can automatically learn such features, no longer relying on predefined ones.

Graph structure features are inductive, meaning that these features are not associated with a particular node or network. For example, the common neighbor heuristic between any pair of nodes $x$ and $y$ is consistently calculated by counting the number of their common one-hop neighbors, invariant to where $x$ and $y$ are located. Thus, graph structure features are transferrable to new nodes and new networks. This is in contrast to latent features, which are often transductive -- the changing of network structure will require a complete retraining to get the latent features again.

\newcolumntype{L}[1]{>{\raggedright\let\newline\\\arraybackslash\hspace{0pt}}m{#1}}
\begin{table}[h]\renewcommand{\arraystretch}{1.5}
\centering
\small
\caption{Popular heuristics for link prediction, see \cite{liben2007link} for details.}
\begin{tabular}{lll}
\toprule
Name& Formula &Order\\ 
\midrule
common neighbors & $|\Gamma(x) \cap \Gamma(y)|$& first\\
Jaccard & $\frac{|\Gamma(x) \cap \Gamma(y)| }{  |\Gamma(x) \cup \Gamma(y)|}$ & first\\ 
preferential attachment & $|\Gamma(x)| \cdot |\Gamma(y)|$& first\\ 
Adamic-Adar & $\sum_{z\in \Gamma(x) \cap \Gamma(y)} \frac{1}{\log|\Gamma(z)|}$ & second\\ 
resource allocation & $\sum_{z\in \Gamma(x) \cap \Gamma(y)} \frac{1}{|\Gamma(z)|}$ & second\\ 
Katz & $\sum_{l = 1}^{\infty} \beta^l {|\text{walks}^{\langle l\rangle}(x,y)|}$ & high\\ 
PageRank & $[\pi_{x}]_y + [\pi_{y}]_x$ & high\\ 
SimRank & $\gamma \frac{\sum_{a\in \Gamma\!(x)}\!\!\sum_{b\in \Gamma\!(y)} \!\!\text{score}(a,b)}{|\Gamma(x)| \cdot |\Gamma(y)|}$ & high\\
\bottomrule
\end{tabular}
 \flushleft
\footnotesize Notes: $\Gamma(x)$ denotes the neighbor set of vertex $x$. $\beta<1$ is a damping factor. $|\text{walks}^{\langle l\rangle}(x,y)|$ counts the number of length-$l$ walks between $x$ and $y$. $[\pi_{x}]_y$ is the stationary distribution probability of $y$ under the random walk from $x$ with restart, see \cite{brin2012reprint}. SimRank score is a recursive definition. We exclude those heuristics which are simple variants of the above or are proven to be meaningless for large graphs (e.g., commute time \cite{luxburg2010getting}).
\label{heuristics}
\end{table}

\textbf{Latent features} are latent properties or representations of nodes, often obtained by factorizing a specific matrix derived from a network, such as the adjacency matrix or the Laplacian matrix. Through factorization, a low-dimensional embedding is learned for each node. Latent features focus more on global properties and long range effects, because the network's matrix is treated as a whole during factorization. Latent features cannot capture structural similarities between nodes \cite{ribeiro2017struc2vec}, and usually need an extremely large dimension to express some simple heuristics \cite{nickel2014reducing}. 
Latent features are also transductive. They cannot be transferred to new nodes or new networks. They are also less interpretable than graph structure features.

Network embedding methods \cite{perozzi2014deepwalk,tang2015line,grover2016node2vec,hamilton2017inductive,lai2017prune,duran2017learning} have gained great popularity recently. They learn low-dimensional representations for nodes too.
Recently, it is shown  that network embedding methods (including DeepWalk \cite{perozzi2014deepwalk}, LINE \cite{tang2015line}, and node2vec \cite{grover2016node2vec}) implicitly factorize some matrix representation of a network~\cite{qiu2017network}. For example, DeepWalk approximately factorizes $\log (\text{vol}(G)(\frac{1}{T}\sum_{r=1}^T(D^{-1}A)^r)D^{-1}) - \log(b)$, where $A$ is the adjacency matrix of the network $G$, $D$ is the diagonal degree matrix, $T$ is skip-gram's window size, and $b$ is the number of negative samples. For LINE and node2vec, there also exist such matrices. Since network embedding methods also factorize matrix representations of networks, we may regard them as learning more expressive latent features through factorizing some more informative matrices.

\textbf{Explicit features} are often given by continuous or discrete node attribute vectors. In principle, any side information about the network other than its structure can be seen as explicit features. For example, in citation networks, word distributions are explicit features of document nodes. In social networks, a user's profile information is also explicit feature (however, their friendship information belongs to graph structure features).

These three types of features are largely orthogonal to each other. Many papers have considered using them together for link prediction \cite{koren2008factorization,rendle2010factorization,nickel2014reducing,zhao2017leveraging} to improve the performance of single-feature-based methods.

\section{More discussion about node labeling}\label{appendix:labeling}
The necessity of structural node labels for enclosing subgraphs is because, unlike ordinary graphs, enclosing subgraphs intrinsically have a directionality. The center of an enclosing subgraph are two nodes $x$ and $y$ between which the target link is located. Outward from the center, other nodes have larger and larger distance to $x$ and $y$. Node labeling is to mark such structural differences thus providing additional structural information to facilitate GNN training.

When designing a node labeling for enclosing subgraphs, we always want to ensure that the target nodes $x$ and $y$ have a distinct label so that GNN can distinguish the target link to predict from other edges. Secondly, we want the node labels to reflect nodes' relative positions in their enclosing subgraph. This relative position can be intuitively described by a node $i$'s double-radius with respect to $x$ and $y$, i.e., $(d(i,x),d(i,y))$.

We restate our Double-Radius Node Labeling (DRNL) algorithm here. First, assign label 1 to $x$ and $y$. Then, for any node $i$ with $(d(i,x), d(i,y)) = (1,1)$, assign label $f_l(i)=2$. Nodes with double-radius $(1,2)$ or $(2,1)$ get label 3. Nodes with double-radius $(1,3)$ or $(3,1)$ get 4. Nodes with $(2,2)$ get 5. Nodes with $(1,4)$ or $(4,1)$ get 6. Nodes with $(2,3)$ or $(3,2)$ get 7. So on and so forth. Our DRNL not only satisfies the above criteria, but also attains the additional benefits that for nodes $i$ and $j$:

1) if $d(i,x) + d(i,y) \neq d(j,x) + d(j,y)$, then $d(i,x) + d(i,y) < d(j,x) + d(j,y) \Leftrightarrow  f_l(i) < f_l(j)$;

2) if $d(i,x) + d(i,y) = d(j,x) + d(j,y)$, then $d(i,x)d(i,y) < d(j,x)d(j,y) \Leftrightarrow  f_l(i) < f_l(j)$.  

That is, the magnitude of node labels also reflects their distance to the center. Nodes with smaller arithmetic mean distance to the target nodes get smaller labels. If two nodes have the same arithmetic mean distance, the node with a smaller geometric mean distance to the target nodes gets a smaller label. Note that these additional benefits will not be available under one-hot encoding of node labels, since the magnitude information will be lost after one-hot encoding. However, such a labeling is potentially useful when node labels are directly used for training, or used to rank the nodes. Furthermore, our node labeling has a perfect hashing (\ref{hashing}) which allows closed-form computation.

We present a lookup table for DRNL and an example labeled subgraph in Figure \ref{nl}. Note that when calculating $d(i,x)$, we temporally remove $y$ from the subgraph, and vice versa. This is because we aim to use the pure distance between $i$ and $x$ without the influence of $y$. If we do not remove $y$, $d(i,x)$ will be upper bounded by $d(i,y) + d(x,y)$, obscuring the ``true distance'' between $i$ and $x$. 

\begin{figure}[tp]
\centering
\subfloat{\includegraphics[width=0.5\textwidth]{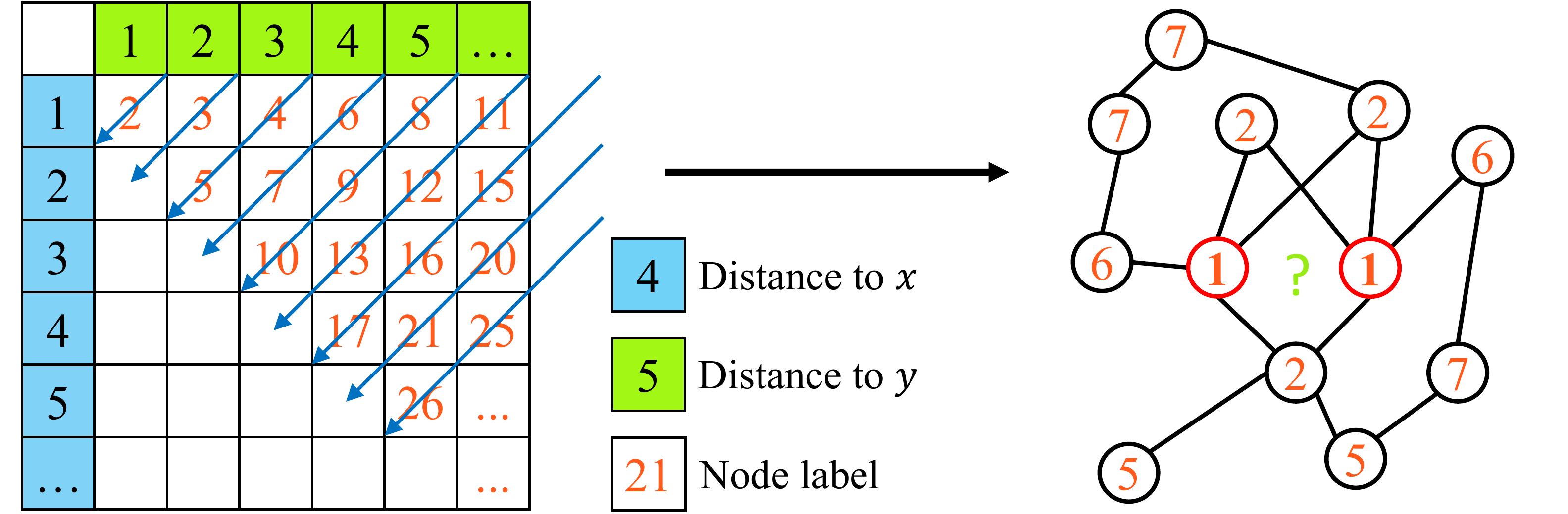}}
\caption{\small Double-Radius Node Labeling.}
\label{nl}
\end{figure}

Our node labeling algorithm is different from the Weisfeiler-Lehman algorithm used in WLNM \cite{zhang2017weisfeiler}. In WLNM, node labeling is for defining a node order in adjacency matrices -- the labels are not really input to machine learning models. To rank nodes with least ties, the node labels should be as fine as possible in WLNM. In comparison, the node labels in SEAL need not be very fine, as their purpose is for indicating nodes' different roles within the enclosing subgraph, not for ranking nodes. 
In addition, node labels in SEAL are encoded into node information matrices and input to machine learning models.

\section{Dataset details}\label{appendix:datasets}

USAir \cite{vladimir2006} is a network of US Air lines with 332 nodes and 2,126 edges. The average node degree is 12.81. NS \cite{newman2006finding} is a collaboration network of researchers in network science with 1,589 nodes and 2,742 edges. The average node degree is 3.45. PB \cite{ackland2005mapping} is a network of US political blogs with 1,222 nodes and 16,714 edges. The average node degree is 27.36. Yeast \cite{von2002comparative} is a protein-protein interaction network in yeast with 2,375 nodes and 11,693 edges. The average node degree is 9.85. C.ele \cite{watts1998collective} is a neural network of C. elegans with 297 nodes and 2,148 edges. The average node degree is 14.46. Power \cite{watts1998collective} is an electrical grid of western US with 4,941 nodes and 6,594 edges. The average node degree is 2.67. Router \cite{spring2004measuring} is a router-level Internet with 5,022 nodes and 6,258 edges. The average node degree is 2.49. E.coli \cite{zhang2018beyond} is a pairwise reaction network of metabolites in E. coli with 1,805 nodes and 14,660 edges. The average node degree is 12.55.

\section{Additional details about baselines} \label{appendix:baselines}
\noindent\textbf{Hyperparameters of heuristic and latent feature methods~~}
Most hyperparameters are inherited from the original paper of each method. For Katz, we set the damping factor $\beta$ to 0.001. For PageRank, we set the damping factor $\alpha$ to 0.85. For SimRank, we set $\gamma$ to 0.8.
For stochastic block model (SBM), we use the implementation of \cite{aicher2015learning} using a latent group number 12. For matrix factorization (MF), we use the libFM \cite{rendle2012factorization} software with the default parameters. 
For node2vec, LINE, and spectral clustering, we first generate 128-dimensional embeddings from the observed networks with default parameters of each software. Then, we use the Hadamard product of two nodes' embeddings as a link's embedding as suggested in \cite{grover2016node2vec}, and train a logistic regression model with Liblinear \cite{fan2008liblinear} using automatic hyperparameter selection. For VGAE, we use its default setting.

\noindent\textbf{WLNM~~} Weisfeiler-Lehman Neural Machine (WLNM) \cite{zhang2017weisfeiler} is a recent link prediction method that learns general graph structure features. It achieves state-of-the-art performance on various networks, outperforming all handcrafted heuristics. WLNM has three steps: enclosing subgraph extraction, subgraph pattern encoding, and neural network training. In the enclosing subgraph extraction step: for each node pair $(x,y)$, WLNM iteratively extracts $x$ and $y$'s one-hop neighbors, two-hop neighbors, and so on, until the enclosing subgraph has \textbf{more than} $K$ vertices, where $K$ is a user-defined integer. In the subgraph pattern encoding step, WLNM uses the Weisfeiler-Lehman algorithm to define an order for nodes within each enclosing subgraph, so that the neural network can read different subgraphs' nodes in a consistent order and learn meaningful patterns. To unify the sizes of the enclosing subgraphs, after getting the vertex order, the last few vertices are deleted so that all the truncated enclosing subgraphs have the same size $K$. These truncated enclosing subgraphs are reordered and their fixed-size adjacency matrices are fed into the fully-connected neural network to train a link prediction model. Due to the truncation, WLNM cannot consistently learn from each link's full $h$-hop neighborhood. The loss of structural information limits WLNM's performance and restrict it from learning complete $h$-order graph structure features. Following \cite{zhang2017weisfeiler}, we use $K=10$ (the best performing $K$) in our experiments.

\begin{table}[h]
\caption{\small Comparison of different link prediction methods}
\centering
\resizebox{0.7\textwidth}{!}{
\begin{tabular}{cccccc}
\toprule
  & Heuristics  & Latent features &WLK & WLNM & SEAL                  \\ \midrule
Graph structure features  & Yes  & No  & Yes  & Yes  & Yes \\
Learn from full $h$-hop  & No  & n/a & Yes  & No  & Yes  \\
Latent/explicit features  & No & Yes  &  No  & No  & Yes  \\
Model  & n/a  & LR/inner product & SVM  &  NN  &  GNN\\
\bottomrule
\end{tabular}}
\label{comparison}
\end{table}

\noindent\textbf{WLK~~} Weisfeiler-Lehman graph kernel (WLK) \cite{shervashidze2011weisfeiler} is a state-of-the-art graph kernel. Graph kernels make kernel machines feasible for graph classification by defining some positive semidefinite graph similarity scores. Most graph kernels measure graph similarity by decomposing graphs into small substructures and adding up the pair-wise similarities between these components. Common types of substructures include walks \citep{vishwanathan2010graph,sugiyama2015halting}, subgraphs \citep{costa2010fast,kriege2012subgraph}, paths \citep{borgwardt2005shortest}, and subtrees \citep{shervashidze2011weisfeiler,neumann2016propagation}. WLK is based on counting common rooted subtrees between two graphs. In our experiments, we train a SVM on the WL kernel matrix. We feed the same enclosing subgraphs as in SEAL to WLK. We search the subtree depth from $\{0,1,2,3,4,5\}$ on 10\% validation links. WLK does not support continuous node information, but supports integer node labels. Thus, we feed the same structural node labels from (\ref{hashing}) to WLK too. 

We compare the characteristics of different link prediction methods in Table \ref{comparison}.

\section{Configuration details of SEAL}

In the experiments, we use Deep Graph Convolutional Neural Network (DGCNN) \cite{zhang2018end} as the default GNN engine of SEAL. DGCNN is a recent GNN architecture for graph classification. It has consistently good performance on various benchmark datasets with a single network architecture (avoid hyperparameter tweaking).
DGCNN is equipped with propagation-based graph convolution layers and a novel graph aggregation layer, called SortPooling. We illustrate the overall architecture of DGCNN in Figure \ref{dgcnn}. Given the adjacency matrix $A \in \{1,0\}^{n\times n}$ and the node information matrix $X \in \mathbb{R}^{n\times c}$ of an enclosing subgraph, DGCNN uses the following graph convolution layer:
\begin{equation}
Z = f( \tilde{D}^{-1} \tilde{A} XW),
\label{gconv}
\end{equation}
where $\tilde{A} = A + I$, $\tilde{D}$ is a diagonal degree matrix with $\tilde{D}_{i,i} = \sum_j \tilde{A}_{i,j}$, $W \in \mathbb{R}^{c\times c'}$ is a matrix of trainable graph convolution parameters, $f$ is an element-wise nonlinear activation function, and $Z \in \mathbb{R}^{n\times c'}$ are the new node states. The mechanism behind (\ref{gconv}) is that the initial node states $X$ are first applied a linear transformation by multiplying $W$, and then propagated to neighboring nodes through the propagation matrix $\tilde{D}^{-1} \tilde{A}$. After graph convolution, the $i^{\text{th}}$ row of $Z$ becomes:
\begin{equation}
Z_i = f\big( \frac{1}{|\Gamma(i)|+1} [X_i W + \sum_{j\in \Gamma(i)} X_j W] \big),
\label{gconv2}
\end{equation}
which summarizes the node information as well as the first-order structure pattern from $i$'s neighbors. DGCNN stacks multiple graph convolution layers (\ref{gconv}) and concatenates each layer's node states as the final node states, in order to extract multi-hop node features.

\begin{figure}[tp]
\centering
\subfloat{\includegraphics[width=0.9\textwidth]{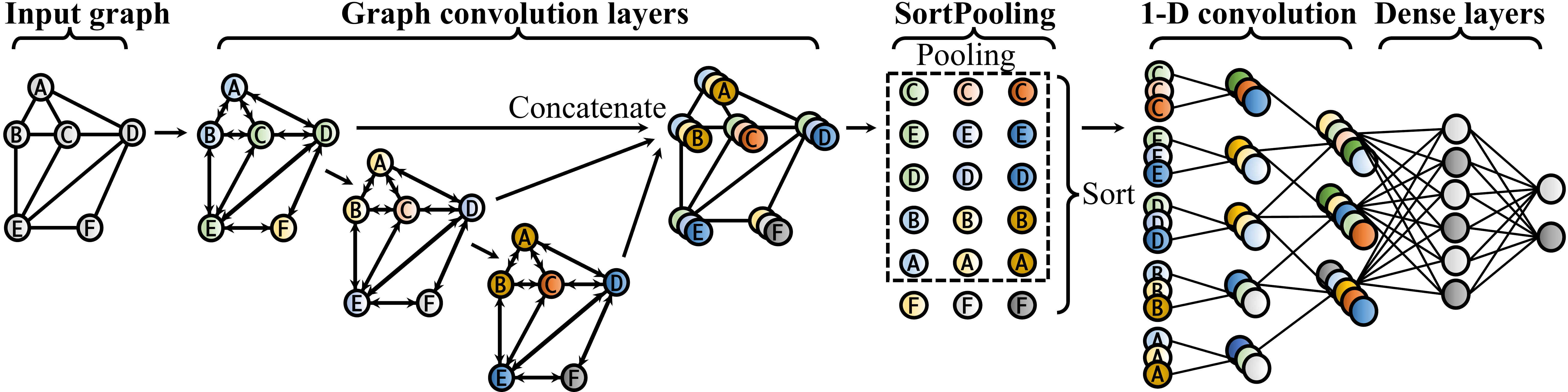}}
\caption{\small The DGCNN architecture.}
\label{dgcnn}
\end{figure}

A graph aggregation layer constructs a graph-level feature vector from individual nodes' final states, which is used for graph classification. The most widely used aggregation operation is summing, i.e., nodes' final states after graph convolutions are summed up as the graph's representation. However, the averaging effect of summing might lose much individual nodes' information as well as the topological information of the graph. DGCNN uses a novel SortPooling layer, which sorts the final node states according to the last graph convolution layer's output to achieve an isomorphism invariant node ordering \cite{zhang2018end}. A max-$k$ pooling operation is then used to unify the sizes of the sorted representations of different graphs, which enables training a traditional 1-D CNN on the node sequence.

We use the default setting of DGCNN, i.e., four graph convolution layers as in (\ref{gconv}) with 32,32,32,1 channels, a SortPooling layer (with $k$ such that 60\% graphs have nodes less than $k$), two 1-D convolution layers (16 and 32 output channels) and a dense layer (128 neurons), see \cite{zhang2018end}. We train DGCNN on enclosing subgraphs for 50 epochs, and select the model with the smallest loss on the 10\% validation data to predict the testing links.

Note that, in any positive training link's enclosing subgraph, we should always remove the edge between the two target nodes before feeding it into a graph classification model. This is because this edge will contain the link existence information, which is not available in any testing link's enclosing subgraph.

%

%

\section{Additional results}\label{appendix:results}

In this section, we show the additional experimental results. We first use 90\% observed links as training links and 10\% as testing links following the main paper's experiments. The average precision (AP) comparison results with heuristic methods are shown in Table \ref{t6}. The AP comparison results with latent feature methods are shown in Table \ref{t8}. We can see that our proposed SEAL shows great performance improvement over all baselines in both AUC and AP.

\begin{table*}[th]
\centering
\caption{\small Comparison with heuristic methods (AP), 90\% training links.}
\resizebox{1\textwidth}{!}{
\begin{tabular}{@{}l |c*{12}{c}}
\toprule
\textbf{Data} &CN     & Jaccard   & PA     & AA     & RA     & Katz    & PR     & SR  &ENS   & WLK            & WLNM            & \textbf{SEAL}     \\ \midrule
USAir  & 93.45$\pm$1.19 & 87.54$\pm$2.07 & 91.22$\pm$1.28 & 95.36$\pm$1.00 & 96.27$\pm$0.79 & 94.07$\pm$1.18 & 95.08$\pm$1.16 & 69.24$\pm$2.61 & 91.33$\pm$1.27 & \textbf{96.82}$\pm$0.84 & 95.95$\pm$1.13 & \textbf{96.80}$\pm$0.55 \\
NS     & 94.39$\pm$0.96 & 94.44$\pm$0.93 & 72.85$\pm$1.88 & 94.46$\pm$0.93 & 94.46$\pm$0.93 & 95.05$\pm$1.08 & 95.11$\pm$1.04 & 94.98$\pm$1.02 & 97.68$\pm$0.36 & 98.79$\pm$0.40 & 98.81$\pm$0.49 & \textbf{99.06}$\pm$0.37 \\
PB     & 91.47$\pm$0.45 & 84.78$\pm$0.71 & 89.33$\pm$0.72 & 92.36$\pm$0.46 & 92.37$\pm$0.57 & 93.07$\pm$0.46 & 92.97$\pm$0.77 & 64.33$\pm$0.95 & 89.35$\pm$0.71 & 93.34$\pm$0.89 & 92.69$\pm$0.64 & \textbf{94.31}$\pm$0.56 \\
Yeast  & 89.34$\pm$0.62 & 89.15$\pm$0.67 & 85.36$\pm$0.85 & 89.53$\pm$0.63 & 89.55$\pm$0.63 & 95.23$\pm$0.39 & 95.47$\pm$0.43 & 93.42$\pm$0.64 & 85.54$\pm$0.85 & 96.82$\pm$0.35 & 96.40$\pm$0.38 & \textbf{98.33}$\pm$0.37 \\
C.ele  & 82.62$\pm$1.51 & 77.06$\pm$2.55 & 75.49$\pm$1.86 & 86.46$\pm$1.43 & 87.10$\pm$1.53 & 85.93$\pm$1.69 & \textbf{89.56}$\pm$1.57 & 68.61$\pm$2.31 & 75.69$\pm$1.86 & 88.96$\pm$2.06 & 85.08$\pm$2.05 & 89.48$\pm$1.85 \\
Power  & 58.77$\pm$0.88 & 58.77$\pm$0.89 & 51.93$\pm$1.16 & 58.76$\pm$0.89 & 58.76$\pm$0.90 & 79.82$\pm$0.91 & 80.56$\pm$0.91 & 77.02$\pm$0.93 & 83.63$\pm$1.37 & 83.02$\pm$3.19 & 87.16$\pm$0.77 & \textbf{89.55}$\pm$1.29 \\
Router & 56.39$\pm$0.53 & 55.84$\pm$0.80 & 69.03$\pm$0.95 & 56.50$\pm$0.51 & 56.51$\pm$0.50 & 64.52$\pm$0.81 & 64.91$\pm$0.85 & 58.82$\pm$1.12 & 69.25$\pm$0.96 & 86.59$\pm$2.23 & 93.53$\pm$1.09 & \textbf{96.23}$\pm$1.71 \\
E.coli & 93.49$\pm$0.38 & 82.42$\pm$0.59 & 94.04$\pm$0.33 & 96.05$\pm$0.25 & 96.72$\pm$0.25 & 94.83$\pm$0.30 & 96.41$\pm$0.33 & 55.01$\pm$0.86 & 94.11$\pm$0.33 & 97.25$\pm$0.42 & 97.50$\pm$0.23 & \textbf{98.03}$\pm$0.20 \\ \bottomrule
\end{tabular}
}
\label{t6}
\end{table*}

\begin{table*}[!th]
\centering
\caption{\small Comparison with latent feature methods (AP), 90\% training links.}
\resizebox{0.63\textwidth}{!}{
\begin{tabular}{@{}l |c*{7}c}
\toprule
\textbf{Data}   & MF    & SBM   & N2V   & LINE  & SPC   & VGAE  & \textbf{SEAL } \\ \midrule
USAir  & 94.36$\pm$0.79 & 95.08$\pm$1.10 & 89.71$\pm$2.97 & 79.70$\pm$11.76 & 78.07$\pm$2.92 & 89.27$\pm$1.29 & \textbf{97.13}$\pm$0.80 \\
NS     & 78.41$\pm$3.85 & 92.13$\pm$2.36 & 94.28$\pm$0.91 & 85.17$\pm$1.65  & 90.83$\pm$2.16 & 95.83$\pm$1.04 & \textbf{98.12}$\pm$0.77 \\
PB     & 93.56$\pm$0.71 & 93.35$\pm$0.52 & 84.79$\pm$1.03 & 78.82$\pm$2.71  & 86.57$\pm$0.61 & 90.38$\pm$0.72 & \textbf{94.55}$\pm$0.43 \\
Yeast  & 92.01$\pm$0.47 & 92.73$\pm$0.44 & 94.90$\pm$0.38 & 90.55$\pm$2.39  & 94.63$\pm$0.56 & 95.19$\pm$0.36 & \textbf{97.95}$\pm$0.35 \\
C.ele  & 83.63$\pm$2.09 & 84.66$\pm$2.95 & 83.12$\pm$1.90 & 67.51$\pm$2.72  & 62.07$\pm$2.40 & 78.32$\pm$3.49 & \textbf{88.81}$\pm$2.32 \\
Power  & 53.50$\pm$1.22 & 65.48$\pm$1.85 & 81.49$\pm$0.86 & 56.66$\pm$1.43  & \textbf{91.00}$\pm$0.58 & 75.91$\pm$1.56 & 86.69$\pm$1.50 \\
Router & 82.59$\pm$1.38 & 84.67$\pm$1.89 & 68.66$\pm$1.49 & 71.92$\pm$1.53  & 73.53$\pm$1.47 & 70.36$\pm$0.85 & \textbf{95.66}$\pm$1.23 \\
E.coli & 95.59$\pm$0.31 & 95.30$\pm$0.27 & 90.87$\pm$1.48 & 86.45$\pm$1.82  & 96.08$\pm$0.37 & 92.77$\pm$0.65 & \textbf{97.83}$\pm$0.20 \\ \bottomrule
\end{tabular}
}
\label{t8}
\end{table*}
\vspace{5pt}

To evaluate SEAL's scalability, we show its single-GPU inference time performance in Table \ref{t9}. As we can see, SEAL has good scalability. For networks with over 1E7 potential links, SEAL took less than an hour to make all the predictions. One possible way to further scale SEAL to social networks with millions of users is to first use some simple heuristics such as common neighbors to filter out most unlikely links and then use SEAL to make further recommendations. Another way is to restrict the candidate friend recommendations to be those who are at most 2 or 3 hops away from the target user, which will vastly reduce the number of candidate links to infer for each user and thus further increase the scalability.

\begin{table}[!th]
\centering
\caption{\small Inference time of SEAL.}
\label{t9}
\resizebox{0.99\textwidth}{!}{
\begin{tabular}{l|cccccccc}
\toprule
			& USAir            & NS           & PB         &Yeast       & C.ele          & Power   & Router  & E.coli        \\ \midrule
Number of potential links & 5.49E+04	&1.26E+06	&7.46E+05	&2.82E+06	&4.40E+04	&1.22E+07	&1.26E+07	&1.39E+06 \\
Inference time per link (s) & 6.05E-04	&2.55E-04	&2.04E-04	&3.96E-04	&4.13E-04	&1.35E-04	&2.13E-04 & 2.40E-04 \\
Inference time for all potential links (s) & 31 &	321	&146	&1106	&16	&1640	&2681	&328 \\
\bottomrule
\end{tabular}}
\end{table}
\vspace{5pt}

\begin{table*}[t]
\centering
\caption{\small Comparison with heuristic methods (AUC), 50\% training links.}
\resizebox{1\textwidth}{!}{
\begin{tabular}{@{}l |c*{12}{c}}
\toprule
\textbf{Data} &CN     & Jaccard   & PA     & AA     & RA     & Katz    & PR     & SR  &ENS   & WLK            & WLNM            & \textbf{SEAL}     \\ \midrule
USAir  & 87.93$\pm$0.43 & 84.82$\pm$0.52 & 87.59$\pm$0.50 & 88.61$\pm$0.40 & 88.73$\pm$0.39 & 88.91$\pm$0.51 & 90.57$\pm$0.62 & 81.09$\pm$0.59 & 87.71$\pm$0.50 & 91.93$\pm$0.71 & 91.42$\pm$0.95 & \textbf{93.23}$\pm$1.46 \\
NS     & 77.13$\pm$0.75 & 77.12$\pm$0.75 & 65.87$\pm$0.83 & 77.13$\pm$0.75 & 77.13$\pm$0.75 & 82.30$\pm$0.93 & 82.32$\pm$0.94 & 81.60$\pm$0.87 & 87.19$\pm$1.04 & 87.27$\pm$1.71 & 87.61$\pm$1.63 & \textbf{90.88}$\pm$1.18 \\
PB     & 86.74$\pm$0.17 & 83.40$\pm$0.24 & 89.52$\pm$0.19 & 87.06$\pm$0.17 & 87.01$\pm$0.18 & 91.25$\pm$0.22 & 92.23$\pm$0.21 & 81.82$\pm$0.43 & 89.54$\pm$0.19 & 92.54$\pm$0.33 & 90.93$\pm$0.23 & \textbf{93.75}$\pm$0.18 \\
Yeast  & 82.59$\pm$0.28 & 82.52$\pm$0.28 & 81.61$\pm$0.39 & 82.63$\pm$0.27 & 82.62$\pm$0.27 & 88.87$\pm$0.28 & 89.35$\pm$0.29 & 88.50$\pm$0.26 & 81.84$\pm$0.38 & 91.15$\pm$0.35 & 92.22$\pm$0.32 & \textbf{93.90}$\pm$0.54 \\
C.ele  & 72.29$\pm$0.82 & 69.75$\pm$0.86 & 73.81$\pm$0.97 & 73.37$\pm$0.80 & 73.42$\pm$0.82 & 79.99$\pm$0.59 & \textbf{84.95}$\pm$0.58 & 76.05$\pm$0.80 & 74.11$\pm$0.96 & 83.29$\pm$0.89 & 75.72$\pm$1.33 & 81.16$\pm$1.52 \\
Power  & 53.38$\pm$0.22 & 53.38$\pm$0.22 & 46.79$\pm$0.69 & 53.38$\pm$0.22 & 53.38$\pm$0.22 & 57.34$\pm$0.51 & 57.34$\pm$0.52 & 56.16$\pm$0.45 & 62.70$\pm$0.95 & 63.44$\pm$1.29 & 64.09$\pm$0.76 & \textbf{65.84}$\pm$1.10 \\
Router & 52.93$\pm$0.28 & 52.93$\pm$0.28 & 55.06$\pm$0.44 & 52.94$\pm$0.28 & 52.94$\pm$0.28 & 54.39$\pm$0.38 & 54.44$\pm$0.38 & 54.38$\pm$0.42 & 55.06$\pm$0.44 & 71.25$\pm$4.37 & 86.10$\pm$0.52 & \textbf{86.64}$\pm$1.58 \\
E.coli & 86.55$\pm$0.57 & 81.70$\pm$0.42 & 90.80$\pm$0.40 & 87.66$\pm$0.56 & 87.81$\pm$0.56 & 89.81$\pm$0.46 & 92.96$\pm$0.43 & 73.70$\pm$0.53 & 90.88$\pm$0.40 & 92.38$\pm$0.46 & 92.81$\pm$0.30 & \textbf{94.18}$\pm$0.41 \\ \bottomrule
\end{tabular}
}
\label{t10}
\end{table*}

\begin{table*}[!th]
\centering
\caption{\small Comparison with latent feature methods (AUC), 50\% training links.}
\resizebox{0.63\textwidth}{!}{
\begin{tabular}{@{}l |c*{7}c}
\toprule
\textbf{Data}   & MF    & SBM   & N2V   & LINE  & SPC   & VGAE  & \textbf{SEAL } \\ \midrule
USAir  & 91.28$\pm$0.71 & 91.68$\pm$0.66 & 84.63$\pm$1.58 & 72.51$\pm$12.19 & 65.42$\pm$3.41 & 90.09$\pm$0.94 & \textbf{93.36}$\pm$0.67 \\
NS     & 62.95$\pm$1.03 & 81.91$\pm$1.55 & 80.29$\pm$1.20 & 65.96$\pm$1.60  & 79.63$\pm$1.34 & \textbf{93.38}$\pm$1.07 & 87.73$\pm$1.08 \\
PB     & 93.27$\pm$0.16 & 92.96$\pm$0.20 & 79.29$\pm$0.67 & 75.53$\pm$1.78  & 78.06$\pm$1.00 & 90.57$\pm$0.69 & \textbf{93.79}$\pm$0.25 \\
Yeast  & 84.99$\pm$0.49 & 88.32$\pm$0.38 & 90.18$\pm$0.17 & 79.44$\pm$7.90  & 89.73$\pm$0.28 & \textbf{93.51}$\pm$0.41 & 93.30$\pm$0.51 \\
C.ele  & 78.49$\pm$1.73 & 81.83$\pm$1.44 & 75.53$\pm$1.23 & 59.46$\pm$7.08  & 47.30$\pm$0.91 & 81.51$\pm$1.69 & \textbf{82.33}$\pm$2.31 \\
Power  & 50.53$\pm$0.60 & 57.53$\pm$0.76 & 55.40$\pm$0.84 & 53.44$\pm$1.83  & 56.51$\pm$0.94 & \textbf{70.34}$\pm$0.84 & 61.88$\pm$1.31 \\
Router & 77.49$\pm$0.64 & 74.66$\pm$1.52 & 62.45$\pm$0.81 & 62.43$\pm$3.10  & 53.87$\pm$1.33 & 62.91$\pm$0.95 & \textbf{85.08}$\pm$1.53 \\
E.coli & 91.75$\pm$0.33 & 90.60$\pm$0.58 & 84.73$\pm$0.81 & 74.50$\pm$11.10 & 92.00$\pm$0.50 & 91.27$\pm$0.42 & \textbf{94.17}$\pm$0.36 \\ \bottomrule
\end{tabular}
}
\label{t11}
\end{table*}

\begin{table*}[!h]
\centering
\caption{\small Comparison with heuristic methods (AP), 50\% training links.}
\resizebox{1\textwidth}{!}{
\begin{tabular}{@{}l |c*{12}{c}}
\toprule
\textbf{Data} &CN     & Jaccard   & PA     & AA     & RA     & Katz    & PR     & SR  &ENS   & WLK            & WLNM            & \textbf{SEAL}     \\ \midrule
USAir  & 87.60$\pm$0.45 & 80.35$\pm$1.26 & 90.29$\pm$0.45 & 89.39$\pm$0.39 & 89.54$\pm$0.36 & 91.29$\pm$0.36 & 91.93$\pm$0.50 & 73.04$\pm$0.84 & 90.47$\pm$0.45 & 93.34$\pm$0.51 & 92.54$\pm$0.81 & \textbf{94.11}$\pm$1.08 \\
NS     & 77.11$\pm$0.74 & 77.10$\pm$0.75 & 68.56$\pm$0.71 & 77.14$\pm$0.74 & 77.14$\pm$0.75 & 82.69$\pm$0.88 & 82.73$\pm$0.90 & 81.86$\pm$0.88 & 86.77$\pm$0.88 & 89.97$\pm$1.02 & 90.10$\pm$1.11 & \textbf{92.21}$\pm$0.97 \\
PB     & 85.90$\pm$0.16 & 78.59$\pm$0.43 & 88.83$\pm$0.25 & 87.24$\pm$0.18 & 87.05$\pm$0.21 & 91.54$\pm$0.16 & 91.92$\pm$0.25 & 70.78$\pm$0.69 & 88.87$\pm$0.25 & 92.34$\pm$0.34 & 91.01$\pm$0.20 & \textbf{93.42}$\pm$0.19 \\
Yeast  & 82.55$\pm$0.27 & 82.16$\pm$0.39 & 84.45$\pm$0.34 & 82.68$\pm$0.27 & 82.66$\pm$0.27 & 92.22$\pm$0.21 & 92.54$\pm$0.23 & 90.98$\pm$0.30 & 84.77$\pm$0.34 & 93.55$\pm$0.46 & 93.93$\pm$0.20 & \textbf{95.32}$\pm$0.38 \\
C.ele  & 69.82$\pm$0.74 & 64.04$\pm$1.02 & 74.20$\pm$0.65 & 73.40$\pm$0.77 & 73.33$\pm$0.96 & 79.94$\pm$0.79 & \textbf{84.15}$\pm$0.86 & 68.45$\pm$1.17 & 74.62$\pm$0.64 & 83.20$\pm$0.90 & 76.12$\pm$1.08 & 81.01$\pm$1.51 \\
Power  & 53.37$\pm$0.22 & 53.35$\pm$0.24 & 51.44$\pm$0.59 & 53.37$\pm$0.23 & 53.37$\pm$0.23 & 57.63$\pm$0.52 & 57.61$\pm$0.56 & 56.19$\pm$0.49 & 61.81$\pm$0.71 & 63.97$\pm$1.81 & 66.43$\pm$0.85 & \textbf{68.14}$\pm$1.02 \\
Router & 52.91$\pm$0.27 & 52.71$\pm$0.23 & 65.20$\pm$0.42 & 52.94$\pm$0.27 & 52.93$\pm$0.27 & 60.87$\pm$0.26 & 61.01$\pm$0.30 & 58.27$\pm$0.51 & 65.38$\pm$0.42 & 75.49$\pm$3.43 & 86.12$\pm$0.68 & \textbf{87.79}$\pm$1.71 \\
E.coli & 86.42$\pm$0.54 & 78.71$\pm$0.40 & 93.25$\pm$0.26 & 89.01$\pm$0.49 & 89.21$\pm$0.48 & 91.93$\pm$0.35 & 94.68$\pm$0.28 & 63.05$\pm$0.48 & 93.35$\pm$0.27 & 94.51$\pm$0.32 & 94.47$\pm$0.21 & \textbf{95.58}$\pm$0.28 \\ \bottomrule
\end{tabular}
}
\label{t12}
\end{table*}

\begin{table*}[!h]
\centering
\caption{\small Comparison with latent feature methods (AP), 50\% training links.}
\resizebox{0.63\textwidth}{!}{
\begin{tabular}{@{}l |c*{7}c}
\toprule
\textbf{Data}   & MF    & SBM   & N2V   & LINE  & SPC   & VGAE  & \textbf{SEAL } \\ \midrule
USAir  & 92.33$\pm$0.90 & 92.79$\pm$0.44 & 82.51$\pm$2.08 & 71.75$\pm$11.85 & 70.18$\pm$2.16 & 89.86$\pm$1.23 & \textbf{94.15}$\pm$0.54 \\
NS     & 66.62$\pm$0.89 & 84.14$\pm$1.18 & 86.01$\pm$0.87 & 71.53$\pm$0.97  & 81.16$\pm$1.26 & \textbf{95.31}$\pm$0.80 & 90.42$\pm$0.79 \\
PB     & 92.53$\pm$0.33 & 92.64$\pm$0.17 & 77.21$\pm$0.97 & 78.72$\pm$1.24  & 81.30$\pm$0.84 & 90.57$\pm$0.79 & \textbf{93.40}$\pm$0.33 \\
Yeast  & 87.28$\pm$0.57 & 90.65$\pm$0.24 & 92.45$\pm$0.23 & 83.06$\pm$9.70  & 92.07$\pm$0.27 & 94.71$\pm$0.25 & \textbf{94.83}$\pm$0.38 \\
C.ele  & 77.82$\pm$1.59 & 80.52$\pm$0.92 & 72.91$\pm$1.74 & 60.71$\pm$6.26  & 55.31$\pm$0.93 & 79.54$\pm$1.60 & \textbf{81.99}$\pm$2.18 \\
Power  & 52.45$\pm$0.63 & 57.23$\pm$0.85 & 60.83$\pm$0.68 & 55.11$\pm$3.49  & 59.10$\pm$1.06 & \textbf{74.86}$\pm$0.43 & 65.28$\pm$1.25 \\
Router & 81.25$\pm$0.56 & 77.77$\pm$1.13 & 66.77$\pm$0.57 & 64.87$\pm$6.76  & 59.13$\pm$3.22 & 71.25$\pm$0.66 & \textbf{86.70}$\pm$1.59 \\
E.coli & 94.04$\pm$0.36 & 93.17$\pm$0.35 & 85.41$\pm$0.94 & 75.98$\pm$14.45 & 94.14$\pm$0.29 & 93.41$\pm$0.32 & \textbf{95.67}$\pm$0.24 \\ \bottomrule
\end{tabular}
}
\label{t13}
\end{table*}

Next, we redo the comparisons under 50\%--50\% train/test split. We randomly remove 50\% existing links as positive testing links and use the remaining 50\% existing links as positive training links. The same number of negative training and testing links are sampled from the nonexistent links as well. The AUC results are shown in Table \ref{t10} and \ref{t11}. The AP results are shown in Table \ref{t12} and \ref{t13}.

The results are consistent with the 90\%--10\% split setting.
As we can see, SEAL is still the best among all methods in general. The performance gains over heuristic methods are even larger compared to the 90\%-10\% split. This indicates that SEAL is able to learn good heuristics even when the network is very incomplete. SEAL also shows more clear advantages over WLNM. On the other hand, we observe that VGAE becomes a strong baseline when network is sparser by achieving the best AUC results on 3 out of 8 datasets. It is thus interesting to study whether replacing the node2vec embeddings in SEAL with the VGAE embeddings can further improve the performance. We leave it to future work.

\newpage
We further conduct experiments with the setting of the node2vec paper \cite{grover2016node2vec} on five networks: arXiv (18,722 nodes and 198,110 edges) \cite{leskovec2015snap}, Facebook (4,039 nodes and 88,234 edges) \cite{leskovec2015snap}, BlogCatalog (10,312 nodes, 333,983 edges and 39 attributes) \cite{zafaranisocial}, Wikipedia (4,777 nodes, 184,812 edges and 40 attributes) \cite{mahoney2011large}, and Protein-Protein Interactions (PPI) (3,890 nodes, 76,584 edges and 50 attributes) \cite{stark2006biogrid}. For each network, 50\% of random links are removed and used as testing data, while keeping the remaining network connected. For Facebook and arXiv, all remained links are used as positive training data. For PPI, BlogCatalog and Wikipedia, we sample 10,000 remained links as positive training data. We compare SEAL ($h=1$, 10 training epochs) with node2vec, LINE, SPC, VGAE, and WLNM ($K=10$). For node2vec, we use the parameters provided in \cite{grover2016node2vec} if available. For SEAL and VGAE, the node attributes are used since only these two methods support explicit features. 

Table \ref{t14} shows the results. As we can see, SEAL consistently outperforms all embedding methods. Especially on the last three networks, SEAL (with node2vec embeddings) outperforms pure node2vec by large margins. These results indicate that in many cases, embedding methods alone cannot capture the most useful link prediction information, while effectively combining the power of different types of features results in much better performance. SEAL also consistently outperforms WLNM.

\begin{table}[!th]
\centering
\caption{Comparison with network embedding methods (AUC and standard deviation, OOM: out of memory).}
\label{t14}
\resizebox{0.75\textwidth}{!}{
\begin{tabular}{l|cccccc}
\toprule
			& N2V            & LINE           & SPC         &VGAE       & WLNM           & SEAL           \\ \midrule
arXiv       & 96.18$\pm$0.40 & 84.64$\pm$0.03 & 87.00$\pm$0.14 &OOM & 99.19$\pm$0.03  & \textbf{99.40}$\pm$0.14  \\ 
Facebook    & 99.05$\pm$0.07 & 89.63$\pm$0.06 & 98.59$\pm$0.11 &98.21$\pm$0.22& 99.24$\pm$0.03  & \textbf{99.40}$\pm$0.08  \\ 
BlogCatalog & 85.97$\pm$1.56 & 90.92$\pm$2.05 & 96.74$\pm$0.31 &OOM & 96.55$\pm$0.08  &\textbf{98.10}$\pm$0.60  \\ 
Wikipedia   & 76.59$\pm$2.06 & 74.44$\pm$0.66 & 99.54$\pm$0.04 &89.74$\pm$0.18 & 99.05$\pm$0.03 & \textbf{99.63}$\pm$0.05 \\ 
PPI         & 70.31$\pm$0.79 & 72.82$\pm$1.53 & 92.27$\pm$0.22 &85.86$\pm$0.43 & 88.79$\pm$0.38 & \textbf{93.52}$\pm$0.37 \\
\bottomrule
\end{tabular}}
\end{table}

\end{appendices}

\end{document}